\documentclass{tlp}
\usepackage{amsmath}
\usepackage{amssymb} 
\usepackage{thmtools} 
\usepackage{mathtools} 
\usepackage{stmaryrd}
\usepackage{amsmath}
\usepackage{graphicx}
\usepackage{multirow}
\usepackage{collect}
\newtheorem{definition}{Definition}
\newtheorem{theorem}{Theorem}
\newtheorem{lemma}{Lemma}
\newtheorem{corollary}{Corollary}

\begin{document}

\lefttitle{Unwrapping All ReLU Networks}

\jnlPage{1}{8}
\jnlDoiYr{2021}
\doival{10.1017/xxxxx}

\title[Unwrapping All ReLU Networks]{Unwrapping All ReLU Networks}

\begin{authgrp}
\author{\sn{Villani} \gn{Mattia Jacopo}}
\affiliation{King's College London}
\author{\sn{McBurney} \gn{Peter}}
\affiliation{King's College London}
\end{authgrp}

\history{\sub{Safe and Trusted AI Workshop (ICLP)} \rev{xx xx xxxx;} \acc{xx xx xxxx}}

\maketitle

\begin{abstract}
Deep ReLU Networks can be decomposed into a collection of linear models, each defined in a region of a partition of the input space. This paper provides three results extending this theory. First, we extend this linear decompositions to Graph Neural networks and tensor convolutional networks, as well as networks with multiplicative interactions. Second, we provide proofs that neural networks can be understood as interpretable models such as Multivariate Decision trees and logical theories. Finally, we show how this model leads to computing cheap and exact SHAP values. We validate the theory through experiments with on Graph Neural Networks. 
\end{abstract}

\begin{keywords}
ReLU Networks, Graph Neural Networks, SHAP Values, Explainable Artificial Intelligence
\end{keywords}

\section{Introduction}

Theoretical inquiry in neural networks is paramount in understanding the success and limits of these models. By studying the details of the construction and comparing how architectures are related, we can generate explanations that verify the behaviour of the networks. In this paper we extend existing results by \cite{sudjianto2020unwrapping} to the multilinear setting. Their work shows that ReLU networks, i.e. deep neural networks with ReLU activation functions, can be represented exactly through piecewise-linear functions whose local linear models can be found exactly on each region. We then draw a connection with SHAP Values, showing how this decomposition can provide an explicit representation and a fast procedure to compute them. 

\textbf{Related Work.} This paper builds closely on \citep{sudjianto2020unwrapping}. Work by \cite{montufar2014number} presents bounds to the complexity of the neural network in terms of the number of linear regions as a function of neurons. \cite{balestriero2018spline} is concerned with representing families of neural networks, including Convolutional Neural Networks \cite{lecun1998gradient}, as compositions of affine splines.

\textbf{Roadmap.} In Section \ref{sec_geom} we present notation and theoretical background as well as extending the theory to general families of neural networks. In \ref{sec_symb} we formalise work from \cite{aytekin2022neural}. \ref{sec_exp} proves explainability implications of our work in making the computation SHAP values \cite{lundberg2017unified} faster. Proofs to theorems are contained in the Appendices.

\section{Unwrapping Geometric Neural Networks}\label{sec_geom}
The aim of this section is to describe a neural network as piece-wise linear functions, a process that \cite{sudjianto2020unwrapping} refer to as \textit{unwrapping}. After preliminaries, we take the Graph Convolutional Neural Network (GCN), of which Recurrent Neural Networks (RNNs) \cite{rumelhart1986learning} and Convolutional Neural Networks (CNNs) \cite{lecun1998gradient} are special cases, and derive their local linear model decomposition. We further generalise the results to neural networks with tensor contractions and multiplicative interactions, as present in the Long Short Term Memory network \cite{hochreiter1997long}. 

\subsection{Preliminaries}

By feed-forward neural networks we will mean deep neural networks in which the architecture is determined entirely by a composition of linear transformations and element-wise activation functions on the resulting vectors; this we will call a \textit{layer}. Our focus will be on said architectures having as activation function rectified linear units. 

\textbf{A feed-forward neural network} $\mathcal{N}: \mathbb{R}^n \rightarrow \mathbb{R}^m$ is a composition of $L$ parametrised functions, which we refer to as the number of layers, with $\textbf{N}=[n_1,n_2,n_3,...,n_L]$ neurons per layer, such that:
\begin{align*} 
\chi^{(l)} = \sigma (W^{(l-1)}\chi^{(l-1)} + b^{(l-1)})
&= \sigma(z^{(l)}). 
\end{align*}
Here, for a given layer $l \in \{1,...,L\}$, $z^{(l)}$ are referred to as preactivations, $\chi^{(l)}$s as activations and $b^{(l)}$s as layer biases. In particular, ReLU (Rectified Linear Units) is the activation function $\sigma:\mathbb{R} \rightarrow \mathbb{R}$ applied element-wise, given by: 
\[ 
\chi_i^{(l)} = \sigma(z_i^{(l)}) = \max\{ 0, z_i^{(l)}\}. 
\]

For a given neuron $\chi_i^{(l)}, i \in \{1,..., n_l\}$, the binary activation state is a function $s: \mathbb{R} \rightarrow \{0,1\}$. Generally, define an activation state as a function of $s: \mathbb{R}\rightarrow \mathcal{S} = \{0,1,2,..,S\}$ for a collection of states, where $|\mathcal{S}|>2$. The function generates a natural partition of $\mathbb{R}$ by $s^{-1}$. 
Depending on the state of each neuron, we can define an \textbf{activation pattern} which encodes each state as implied by a given input. Every layer will have an activation vector, the collection of which we call the \textit{activation pattern}.

Given an instance $x \in \mathbb{R}^n$ and a feed-forward neural network $\mathcal{N}$ with $L$ layers, each with a number of neurons described by the vector $\textbf{N}$, the activation pattern is an ordered collection of vectors 
\[
\mathbf{P}(x) = \{\mathbf{P}^{(1)}(x), \mathbf{P}^{(2)}(x), ... , \mathbf{P}^{(L)}(x)\}
\]
such that 
\[
\mathbf{P}_i^{(l)}(x) = s(\chi_i^{(l)}) \in \mathcal{S}, \]
where $i\in \{1,..., n_l\}$ are indices for a activations at a layer $l\in \{1,..., L\}.$

The collection of all points yielding the same activation pattern, which can be thought of as \textit{fibers}, we will call the activation region for the network. 

We refer to the \textbf{activation region} $\mathcal{R}^{\mathbf{P}(x)} \subset \mathbb{R}^n$ of the activation pattern as the collection of points $v \in \mathbb{R}^n$ such that
\[
\forall v \in \mathcal{R}^{\mathbf{P}(x)}, \: \mathbf{P}(v) = \mathbf{P}(x).
\]

Importantly, these regions are convex and partition the input space of the neural network \cite{sudjianto2020unwrapping}. This is key for the characterisation of the neural network as a piece-wise linear function: the convex domain allows us to have a description of the activation regions as intersections of half-spaces.

\begin{theorem}\textbf{Local Linear Model of a ReLU Network},\cite{sudjianto2020unwrapping} \label{thm:llms}
Given a feedforward neural network $\mathcal{N} : \mathbb{R}^n \rightarrow \mathbb{R}^m$, with ReLU activation $\sigma$, $L$ layers and neurons in $\mathbf{N}= [n_1,..., n_L]$, the local linear model $\eta^\mathbf{P}(z)$ for the activation region $\mathcal{R}^\mathbf{P}(z)$ of an activation pattern $\mathbf{P}(z)$, with $z \in \mathcal{X}\subset \mathbb{R}^n$, is given by

\[
\eta^\mathbf{P}(x) = w^{\mathbf{P}(z)T} x + b^{\mathbf{P}(z)}, \forall x \in \mathcal{R}^{\mathbf{P}(z)}
\]

where the weight parameter is given by:

\[
w^{\mathbf{P}(z)}= \prod_{h = 1}^L W^{(L+1-h)} D^{L+1-h} W^{(0)},
\]
and the bias parameter is given by
\[
b^{\mathbf{P}(z)} = \sum_{l=1}^{L} \prod_{h = 1}^{L+1-l} W^{(L+1-h)} D^{L+1-h} b^{(l-1)}+b^{(L)},
\]
where 
\[
D^{(l)} = \text{diag}(\mathbf{P}(z)),
\]
is the diagonal matrix of the activation pattern for a given layer $l \in \{1,...,L\}$. 
\end{theorem}

\subsection{Unwrapping Graph and Tensor Neural Networks}

In the case of neural networks with convolutions, which we intend loosely as parametrised matrix or tensor operations with weights, learnable or otherwise, such as Convolutional Neural Networks \cite{lecun1998gradient}, Graph Convolutional Networks \cite{kipf2016semi}, the local linear model decomposition needs to take into account the \textit{weight sharing} scheme that is implied by the convolution. GCNs encompass RNNs and CNNs, meaning that we set convolutional weights of GCNs  to zero in particular ways to achieve networks that fall in the latter classes of architectures. Therefore, decomposing GCNs is in itself a significant result, which we obtain below. 

\begin{definition}[\textbf{Graph Convolutional Neural Network}]
    Given a graph $\mathcal{G} = (V,E)$ with vertex and edge sets $V,E$ respectively, a Graph Convolutional Network (GCN) is a composition of $L$ parametrised layers, with $\textbf{N}=[n_1,n_2,n_3,...,n_L]$ neurons per layer, yielding a function $\mathcal{N}^{\mathcal{G}}: \mathbb{R}^{k \times n} \rightarrow \mathbb{R}^{k \times m}$, where each forward pass is defined by: 

\[ 
\chi^{(l+1)} = \sigma \left( A \cdot \chi^{(l)} W^{(l)} + b\right)
\]

where $k = |V|$ is the number of nodes of the graph, $\chi^{(l)} \in \mathbb{R}^{k \times n_l}$ and $W^{(l)} \in \mathbb{R}^{n_{l-1} \times n_{l}}$ is a weight matrix. Finally, $b \in \mathbb{R}^{ k \times n_{l} }$ is a matrix of biases and $A$ is a graph convolutional operator $A \in \mathbb{R}^{k \times k}$, often taken to be an adjacency matrix or its Laplacian. 
\end{definition}

This definition underscores how the GCN can be viewed as a multilinear variant of the feedforward neural network. Indeed, two operations are applied to the activations of the previous layer: a left and right matrix multiplication. Viewing these as a single linear operation on a vectorised input allows us to decompose the network similarly to how we have done in the feedforward case. This leads us to our main theorem of the section. 

\begin{theorem} \label{thm:graph}
The local linear model of a Graph Convolutional Network at a point $z \in \mathbb{R}^{n \times m}$ is given by 
\[
\eta^{\mathbf{P}(z)}(X) = w^{\mathbf{P}(z)T} vec(X) + b^{\mathbf{P}(z)}, \text{  } \: \forall X \in \mathcal{R}^{\mathbf{P}(z)}
\]

where, 
\[
w^{\mathbf{P}(z)} = \prod ^{L}_{h=1}\left((W^{(L+1-h)} \otimes A^{(L+1-h)}) \odot \mathbf{P}^{(L+1-h)}(z)^T\right)  W^{(0)}
\]
and the bias parameter is given by
\[
b^{\mathbf{P}(z)} = \sum_{l=1}^{L} \prod_{h = 1}^{L+1-l} \left((W^{(L+1-h)} \otimes A^{(L+1-h)}) \odot \mathbf{P}^{(L+1-h)}(z)^T\right) b^{(l-1)}+b^{(L)},
\]
where 
\[
\mathbf{P}^{(l)}(z) \in \{0,1\}^{n_{l-1} \times n_l},
\]
is the matrix encoding the activation pattern of the network at layer $l$.
\end{theorem}

In the above, $\odot$ is the element-wise or Hadamard product. This construction leads to their main result, which we will extend to large classes of networks. We now proceed to a generalised version of this results that allows us to generate decomposition for networks that apply tensor contractions to a distinguished tensor: the output of the previous layer. This result too relies on the vectorisation of the neural network; which enables us to encode the weight sharing scheme in the Kronecker product of matrices. 

Given a collection of matrix contractions on a tensor $\textbf{X} \in \mathbb{R}^{a_1 \times ... \times a_k} $, as represented by $ \llbracket \textbf{X} ; A_1, A_2, ..., A_k \rrbracket $, also known as \textit{Tucker product}, with $A_i \in \mathbb{R}^{a_i \times a_i'}$ acting on the $i$th mode of the tensor, a \textbf{Tensor Neural Network} is a composition of $L$ parametrised layers, with $\textbf{N}=(\textbf{n}_1,\textbf{n}_2,\textbf{n}_3,...,\textbf{n}_L)$ collection of mode vectors for each layer, each with $k_l$ modes and dimensionality given by a vector $\textbf{n}_l = [a^{(l)}_1, ..., a^{(l)}_{k_l}]$ yielding a function $\mathcal{N}^{\mathcal{T}}: \mathbb{R}^{\times \textbf{n}_1} \rightarrow \mathbb{R}^{ \times \textbf{n}_L}$, where we take $\times \textbf{n}_l = a_{1}^{(l)} \times ... \times a_{k_l}^{(l)}$ each forward pass is defined by: 

\[ 
\mathbf{\chi}^{(l+1)} = \sigma \left( \llbracket \mathbf{\chi}^{(l)}; A^{(l)}_1, A^{(l)}_2, ... A^{(l)}_{k_l} \rrbracket  + \textbf{b}\right)
\]

where $k = |V|$ is the number of nodes of the graph, $\chi^{(l)} \in \mathbb{R}^{\times \textbf{n}_l}$ and $A_i^{(l)} \in \mathbb{R}^{a^{(l-1)}_{i} \times a^{(l)}_{i}}$ is a weight matrix. Finally, $\textbf{b} \in \mathbb{R}^{ \times \textbf{n}_{l+1} }$ is a tensor of biases.

\begin{theorem} \label{thm:tensor}
    The local linear model of a Tensor Neural Network at a point $z \in \mathbb{R}^{\times\textbf{n}_1}$ is given by 
\[
\eta^{\mathbf{P}(z)}(X) = w^{\mathbf{P}(z)T}vec(X) + b^{\mathbf{P}(z)}, \text{  } \: \forall X \in \mathcal{R}^{\mathbf{P}(z)}
\]

where, 
\[
w^{\mathbf{P}(z)} = \prod ^{L}_{h=1}\left( \left(\bigotimes_{i \in [k_h]}A^{(h)}_i \right) \odot \textbf{P}^{(L+1-h)}(z)^T\right)  W^{(0)}
\]
and the bias parameter is given by
\[
b^{\mathbf{P}(z)} = \sum_{l=1}^{L} \prod_{h = 1}^{L+1-l} \left(\left(\bigotimes_{i \in [k_h]}A^{(h)T}_i \right) \odot \textbf{P}^{(L+1-h)}(z)^T\right) \textbf{b}^{(l-1)}+\textbf{b}^{(L)},
\]
where $\textbf{P}(z) \in \{0,1\}^{\times \mathbf{n}_l}$ is a tensor encoding the activation pattern of $\mathcal{N}$ at point $z\in \mathbb{R}^{\times \mathbf{n_1}}$, and $X = vec(z)$.

\end{theorem}

This result is be a stepping stone to generalise for arbitrary tensor contractions, beyond the Tucker product, whenever suitable matrizations apply. In particular, these networks and their decomposition can be mapped to transformations of type $f: \mathbb{R}^n \rightarrow \mathbb{R}^m$, for suitable choices of $n, m \in \mathbb{N}$. The significance is two-fold: on one side, we may understand all higher order architectures as special instances of feed-forward neural networks in which weights are constrained by the Kronecker product scheme. This in turn highlights how these architectures are at best as expressive as neural networks of suitable dimension. 

While many layers can be recovered as a special case of the graph neural network, there are certain layer types, such as the Long Short Term Memory cell \cite{hochreiter1997long}, which involve the point-wise multiplication of two layer outputs. To that end, we show how the decomposition of a multiplicative interaction leads to higher order forms, instead of linear models.

\begin{corollary} \label{prop:mult}
Let multiplicative interactions be defined as the element-wise multiplication of two forward pass layers of neural networks, in the form below:
\[
\chi^{(l+1)} = \sigma( W \chi_1^{(l)} +b ) \odot \sigma( V \chi_2^{(l)} + c).
\]
For a given pair $\chi_1^{(l)}, \chi_2^{(l)}$, there exists a decomposition of the layer given by: 

\[
 \chi^{(l+1)} = D^{(l)}_1 W \chi^{(l)} \odot D_2^{(l)} \chi_2^{(l)} + b \odot D_2^{(l)} \chi_2^{(l)} + c \odot D^{(l)}_1 W \chi^{(l)} + b\odot c
\]
where $D^{(l)}_1, D^{(l)}_2$ are diagonal matrices storing the activation pattern in their diagonal. 
\end{corollary}

\section{Symbolic Representation of Neural Networks} \label{sec_symb}
In this Section we explore the consequences of the decomposition for a symbolic interpretation of the neural network. Indeed, the decomposition opens many paths to inspect the inner workings of the network, but two analogies are particularly fitting. By viewing every activation pattern as a leaf on a tree-based model, we can generate a surrogate that mimicks the behaviour of the neural network exactly. There are several models that can be used, for example \cite{aytekin2022neural} uses general decision trees and \cite{schluter2023towards} use Algebraic Decision Structures. We decide to use Multivariate Regression Trees, as these are easiest to define and resemble most closely the propagation of information in the network. Importantly, all of these models are white-box: computing the tree-based alternative allows us to fully comprehend the global behaviour of the network. 

The second observation is that the half-spaces of the neural network induced by the network form a Boolean algebra in the input space. There is a close link between Boolean algebras and logic, which entails that we can understand the network's functioning as the evaluation of propositions in a first order logic. We will state the formal result after stating the conditions for activations of a given neuron. 

\subsection{Half Space Conditions}
For a ReLU network, every neuron of the first hidden defines a half plane in the input space as follows. 
Let $W \in \mathbb{R}^n$ and $b \in \mathbb{R}$. 
$H(W,b)$ is the half space defined by all $x \in \mathbb{R}^n$ such that: 
\[
W^T\cdot x + b > 0. 
\]
We can see this applied to neural networks. Let, 
\[
\chi^{(1)} = \sigma(W^{(1)} x + b^{(1)}) = \max(W^{(1)} x + b^{(1)}, 0)
\]
then, for all $i \in \{1,...,n_1\}$ where $n_1$ is the number of neurons in the first layer, we have that
\[
s(\chi^{(1)}_i) = 1 \iff x \in H(W_i^{(1)}, b_i^{(1)})
\]
and 
\[
s(\chi^{(1)}_i) = 0 \iff x \not\in H(W_i^{(1)}, b_i^{(1)}).
\]
This implies that a given activation pattern for the first layer $\textbf{P}^{(1)}$, there is an intersection of space, $\omega_{\textbf{P}^{(1)}}$ given by:
\begin{equation}\label{eq:omega}
\omega_{\textbf{P}^{(1)}} = \bigcap^{n_1}_{i = 1} H(W_i^{(1)} \cdot (2 P^{(1)}_i -1) , b^{(1)}_i).
\end{equation}
By iterating the recursion $\chi^{(l+1)} = \max \{W^{(l)}\chi^{(l)} + b^{(l)},0\}$, we provide rules for activation of each neuron $P_i^{(l)}, i \in \{1,...,n_l\}$. This results in the following lemma. 

\begin{lemma}[\textbf{Conditions for Activation}]\label{condact}
Given a neural network $\mathcal{N}: \mathbb{R}^n \rightarrow \mathbb{R}^m$, with $L\in \mathbb{N}$ layers, and an activation pattern $\textbf{P} = \{\textbf{P}^{(1)}, \textbf{P}^{(2)}, ..., \textbf{P}^{(L)} \}, $ the region $\omega_{\textbf{P}} \subset \mathbb{R}^n$ defined by the activation pattern is given by: 
\begin{align*}
\omega_{\textbf{P}} &= \bigcap_{j = 1}^L \bigcap_{i = 1}^{n_j} H \left( w^{\textbf{P}[j]} W_i^{(j)} p^{(j)}_i , b_i^{\textbf{P}[j]} \right) \\ 
\end{align*}
where $p^{(j)}_i$ is an identity matrix with the $i$th diagonal is replaced by $2\cdot P^{(j)}_i-1$, which is the activation state of the $i$th neuron of the $j$th layer, $D^{(h)}$ is the diagonal matrix associated with the activation pattern $\textbf{P}^{(h)}$, given by $D^{(h)} = \text{diag}(\textbf{P}^{(h)})$ for the $h$th layer, and \newline $\textbf{P}[j] = \{\textbf{P}^{(1)},\textbf{P}^{(2)},..., \textbf{P}^{(j)}\}\subset \textbf{P}$, is the set of all activation vectors until layer $j$, so that $w^{\textbf{P}[j]}$ and $b^{\textbf{P}[j]}$ represent the coefficients of a local linear model up to layer $j$, given by \ref{thm:llms}. 
\end{lemma}

Note that for the first layer each neuron represents a half-space, while for the second layer, each neuron can represent a collection of half-spaces, depending on the previous selection. We can understand this as a \textit{hierarchy of concepts}. The first partition $\Omega^{(1)} = \{\omega_{\textbf{P}[1]} : \textbf{P}[1] \in \{0,1\}^{n_1}\}$, defines a set of concepts, which we refine through distinctions represented by the neurons of the second layer. 

In fact, if there are $2^{n_2}$ possible activation patterns, that would define a partition for each of the $2^{n_1}$ contexts. However, some of these activation patterns may define empty regions, and this would depend by the context, i.e. the activation pattern of the previous layers.

\subsection{Networks are Trees and Theories} \label{sub:treentheories}
The hierarchical description implied by the recursive partitioning of the neural networks' layers motivate the relationship between neural networks and tree based models. In particular, we search a models that can replicate the behaviour of the neural network exactly: every path representing the conditions applied by a given activation pattern and each leaf containing the data of the linear model we will apply in that case. Indeed, this description refers to the Multivariate Regression Tree \cite{de2002multivariate}, which we define in the appendix. We fit the data of the neural network, empowered by the decomposition into local linear models. 

\begin{theorem}[\textbf{Tree for a Neural Network}] \label{thm:tree}
For every feedforward ReLU neural network $\mathcal{N}$ there exist a MRT $(\mathcal{M}, T, e, \Theta)$ that represents exactly the behaviour of the neural network: 
\[ \mathcal{M}(x) = \mathcal{N}(x) \text{    } \forall x \in \mathbb{R}^n. \]
\end{theorem}

This result is important insofar as it allows us to represent a neural reasoner symbolically. In particular, it proves the observations of \cite{aytekin2022neural} formally. A challenge is to find and store all the partitions. 

We can consider individual half-space divisions as atoms in a propositional logic. The following comments reflect the spirit of \cite{schluter2023towards}, who prove a correspondence between ReLU networks and algebraic decision structure. We show instead that there is an internal logic to the neural network, which can be computed by the half-space algebra. 

\begin{corollary}[\textbf{Internal Logic of a Network}] \label{prop:theory}A ReLU feedforward neural network $\mathcal{N}: \mathbb{R}^n \rightarrow \mathbb{R}^m$ induces a Boolean algebra which is the Lidenbaum-Tarski algebra of a theory $\mathbb{T}$ in classical propositional logic given by: 
\begin{itemize}
    \item A collection of propositional variables $h^{(l)}_i, l \in [L], i \in [n_l]$ 
    \item A collection of terms determined by arbitrary meets $\textbf{P} = \{ p : p = \bigwedge_{l \in [L], i \in [n_{l}] } h^{(l)}_i\}$,
    \item Axioms and formulas pursuant the structure of the Boolean algebra. 
\end{itemize}
\end{corollary}
This follows directly from assigning to each variable the truth statement of $x \in H \left( w^{\textbf{P}[j]} W_i^{(j)} p^{(j)}_i , b_i^{\textbf{P}[j]} \right)$ as defined in Lemma \ref{condact}, for all possible activation patterns. Then, the Boolean algebra spanned by the half spaces implied by the neural network activations returns the required theory. 

Half-space conditions are the alphabet of the neural network's reasoning, meaning that propositions are then formed by taking arbitrary intersections of these conditions. There are two important consequences. Transformations of neural networks imply transformations of their underlying grammar: by applying backpropagation we obtain a morphism of trees and theories in adequate categories. This justifies the intention to use category theory as an instrument to analyse the interplay between architecture and representation system \cite{spivak2021learners}. 

\section{Explainability} \label{sec_exp}
In this section we get the SHAP Values for ReLU neural networks explicitly. We can use the previous results to compute exact local Shapley values for an instance. Precisely, given that we have an explicit local model for each region $\mathcal{R}^{\mathbf{P}(z)}$ we can state the following. 

\subsection{SHAP Values}
We recall the definition of SHAP values from \cite{lundberg2017unified} on a given function $f$. 
\begin{definition}[\textbf{SHAP Values}]
    Given a function $f:\mathbb{R}^n \rightarrow \mathbb{R}^m$, the SHAP values for feature $i \in [n]$ are given by:
    \[\phi_v(i) = \sum_{S \in \mathcal{P}([n]) / \{i\} } \frac{(N-|S| +1)|S|!}{N!} \Delta_v(S,i),\]
    where $\mathcal{P}([n])$ is the set of a all subsets of $[n]$, $v:\mathcal{P}([n]) \rightarrow \mathbb{R}$ is a value function, and $\Delta_v(S,i)$ is the marginal contribution of a feature $i$ on a subset $S \in \mathcal{P}([n])$, which we refer to as a \textit{coalition} of features, is given by: 
    \[ \Delta_v (i,S) = v(\{i\} \cup S) - v(S).\]
\end{definition}

These provide concrete examples of how a piece-wise linear theory of architecture can support development of XAI techniques.

\begin{lemma}[\textbf{Local Shapley Values of a ReLU Neural Network}]\label{thm:shap_local}
Given a neural network $\mathcal{N}: \mathbb{R}^n \rightarrow \mathbb{R}^m$, with hyperparameters $L$, the number of layers and $\mathbf{N}= [n_1,...,n_L]$ the number of neurons for each layer, given a Linear Local Model decomposition with $\eta^{\mathbf{P}}(x)$ for $x \in \mathcal{R}^{\mathbf{P}(z)}$ with activation pattern $\mathbf{P}(z)$, the Shapley value is given by: 
\[\phi_f(i)_j = \tilde{w}_{i,j}^{\mathbf{P}(z)}(x_i - \bar{x_i}) \]
so long as $\bar{x}_S \in \mathcal{R}^{\mathbf{P}(z)}$ for all coalitions, $, \forall S \in \mathcal{P}([n]/i)$.
\end{lemma}

This theorem implies that given a neural network decomposition, exact SHAP values can be computed simply by finding the linear model for instance of interest and its masked counterparts. Summing the coefficients according to the above formula will return the desired value. This entails a reduction in computation time as we are no longer fitting local surrogates as in KernelSHAP, whenever the decomposition is readily available. In particular, these SHAP values are \textit{exact}, meeting an increasing need for faithfulness of explanations, both in practice and for regulatory purposes. We also prove a global version of this Theorem, with weaker assumptions, which can be found in the Appendix.



\bibliographystyle{plainnat}
\bibliography{bibliography.bib}


\appendix

\section{Proof of GNN Decomposition}
\begin{proof}
First we realise that, whenever $A, X, B$ are compatible matrices, we have:
\[
vec(AXB) = (B^T \otimes A) \cdot vec(X), 
\]
also known as the \textit{vec trick}. The key in the computation of the weights is noticing that the layerwise propagation function can be represented by this vectorisation,
\begin{align*}
\chi^{(l)} & = \sigma (A^{(l-1)}\chi^{(l-1)}W^{(l-1)}) \\
& = \sigma((W^{(l-1)T} \otimes A^{(l-1)}) vec(\chi ^{(l-1)})). \\
\end{align*}
Using this fact allows us to treat the matrix $\tilde{W}^{(l-1)} = (W^{(l-1)}\otimes A^{(l-1)})$ as the weight matrix of a linear neural network. To encode the activation pattern of the graph neural network we take the activation pattern, which in this case is a matrix $D^{(l)}$. Combining the fact that

\begin{align*}
vec(A \odot B) = vec(A) \odot vec(B),
\end{align*}

with the vec trick results in:  

\begin{align*}
w^{\mathbf{P}(z)} & = \prod ^{L}_{h=1}\left((W^{(L+1-h)}  \otimes A^{(L+1-h)} ) \odot \mathbf{P}^{(L+1-h)}(z))\right) W^{(0)},
\end{align*}
and this exact replacement produces the bias parameters. 
\end{proof}
\section{Proof of TCN Decomposition}
\begin{proof}
The key step in this proof it to realise that the vector trick generalises to higher dimensional contractions between a tensor and a matrix. 

\begin{align*}
vec(\chi^{(l+1)}) & = \sigma \left( vec( \llbracket \chi^{(l)}; A_1^{(l)}, A_2^{(l)},...,  A_{k_l}^{(l)}\rrbracket + vec(\textbf{b}) \right) \\
& = \sigma \left( \bigotimes_{i = 1}^{k_l} A^{(l)T}_i vec(\chi^{(l)}) + vec(\textbf{b}) \right) \\
& = \sigma \left( \tilde{A}^{(l)} vec(\chi^{(l)}) + vec(\textbf{b}),  \right) \\
\end{align*}
and with the vectorisation of the Hadamard product being preserved element-wise, the proof follows exactly the replacement in the equivalent proof for the Graph Convolutional Network. 
\end{proof}

\section{Proof of Multiplicative Interaction Decomposition}
\begin{proof}
As before, we realise that there exist a diagonal matrices $D_1, D_2$ that hold the activation pattern and such that preserve locally the behaviour of ReLU activation is replicated. We observe that: 
\begin{align*}
    \chi^{(l+1)} &  = (D^{(l)}_1 W \chi^{(l)} + b) \odot (D_2^{(l)} \chi_2^{(l)} +c)  \\
    & = D^{(l)}_1 W \chi^{(l)} \odot D_2^{(l)} \chi_2^{(l)} + b \odot D_2^{(l)} \chi_2^{(l)} + c \odot D^{(l)}_1 W \chi^{(l)} + b\odot c, 
\end{align*}
by distributivity of the Hadamard product. 
\end{proof}
\section{Proof of Conditions of Activations for a Neural Network}
\begin{proof}
Set an activation pattern given by $\textbf{P} = \{ \textbf{P}^{(1)}, \textbf{P}^{(2)}, ..., \textbf{P}^{(L)} \}$ where $\textbf{P}^{(l)} \in \{0,1\}^{n_l}, l \in \{1,...,L\}$. We prove the statement using mathematical induction on the number of layers $L$. 

For the case $L = 1$, with $\textbf{P} = \{P^{(1)}\}$, the region is given by equation \ref{eq:omega}, which is the only layer. For the inductive step, assuming that case $L = l$ is true, we prove it implies the formula for $L = l+1$. Recall that $n_{l+1} \in \mathbb{N} , \mathbf{P}^{(l+1)} \in \{0,1\}^{n_{L+1}}$, the activation pattern for the $l+1$th layer. We can think of each neuron in the subsequent $l+1$th layer as imposing a further restriction on the polytope defined by $\omega_{\mathbf{P}[l]}$, leading to a collection of half spaces given by:   
\[
\bigcap_{i = 1}^{n_l} H(W_i^{(l+1)}p_i^{(l+1)}, b_i^{(l+1)}) \subset \mathbb{R}^{n_l}.
\]
However, we stress these half spaces live in $\mathbb{R}^{n_{l}}$. To express the conditions on $\mathbb{R}^n$, the input space, we project back the half space into the domain of the previous layers recursively. In particular, let $\chi^{(l+1)}$ be the post-activation of the $l$th layer. Then, expanding by the linear model decomposition, we obtain the following equivalent conditions. 

\begin{align*}
 & \chi^{(l+1)} \in H(W^{(l+1} p_i^{(l+1)}, b^{(l+1)_i}) \\
\iff & (W_i^{(l+1)} p_i^{(l+1)})^T \chi^{(l+1)} + b^{(l+1)}_i  > 0 \\
\iff & (W_i^{(l+1)} p_i^{(l+1)T}) D^{(l)} (W^{(l)T} \chi^{(l)} + b^{(l)}) + b_i^{(l+1)} > 0 \\
\iff & (W_i^{(l+1)} p_i^{(l+1)})^T D^{(l)} W^{(l) T}\chi^{(l)} +  (W_i^{(l+1)} p_i^{(l+1)})^T D^{(l)} b^{(l)} + b_i^{(l+1)} > 0 \\
& ... \\
\iff & (W_i^{(l+1)} p_i^{(l+1)} w^{\textbf{P}[l]})^T x + b^{\textbf{P}[l]}_i + b^{(l+1)}_i > 0 \\
\iff & x \in H \left( w^{\textbf{P}[j]} W_i^{(j)} p^{(j)}_i , b_i^{\textbf{P}[j]} \right).
\end{align*}
Taking the intersection for all $i \in \{1,..., n_l\}$ and across layers provides us with the desired result. 
\end{proof}

\section{Proof of Existence of Multivariate Regression Tree for Every Neural Network}
We define the Multivariate Regression tree and prove the statement of the theorem. 
\begin{definition}[\textbf{Multivariate Regression Tree}]
For a learning problem $\mathcal{D} = \mathcal{X} \times \mathcal{Y}$, where $\mathcal{X} \subset \mathbb{R}^n, \mathcal{Y} \subset \mathbb{R}^m$, a Multivariate Regression Tree (MRT) is a tuple $(\mathcal{M}, T, e, \Theta)$ where
\begin{itemize}
    \item $T = (V,E)$  is a binary tree,
    \item $e: E \rightarrow \mathbb{R}^n \times \mathbb{R}$ are edge labels representing the half space conditions $H(W,b), \neg H(W,b)$ imposed by each bifurcation of the tree,
    \item $\Theta: \Lambda \rightarrow \mathbb{R}^{n\times m} \times \mathbb{R}^m$ is a function that assigns to each leaf $\lambda \in \Lambda \subset \mathcal{P}(V)$ (identified as the unique path from the root) parameters for a linear model, and
    \item $\mathcal{M}:\mathbb{R}^n \rightarrow \mathbb{R}^m$ is a function that applies for every $x\in \mathcal{R}^n$ the linear model

\end{itemize}  
\[
\eta^{\lambda} = (W^\lambda) ^T \cdot x + b^\lambda , x \in \mathbb{R}^n,
\]
whenever $x \in \bigcap_{e \in \lambda} H(e_1(E),e_2(E))$, the collection of half-spaces imposed by the path, where $e_1,e_2$ are the two components of $e$. 
\end{definition}

\begin{proof}
$\mathcal{N}$ has $L$ layers, each with $n_i, i \in [L]$ neurons. Each of these neurons provides a halfspace, as given by \ref{condact}. 
Therefore, each architecture $\mathbf{N}$ dictates a tree $T$ with $V \cong \bigcup_{i \in [L]} [n_i]$, 
the set of vertices is one-to-one with the set of all neurons, and 
\[ \mathcal{P}(E) \cong \prod_{i \in [L]} \{0,1\}^{n_i}, \] 
where in particular we map the activation pattern $\mathbf{P}(z)$ to a path $ \lambda $ by building 
$e(a^{(1)}_{i,j}) = (w^{\mathbf{P}[j]} W_i^{(j)} p_i^{(j)}, b^{\mathbf{P}[j]})$, where $a^{(1)}_{i,j}$ is the edge representing the $i$th neuron of the $j$th layer being active, and the components of the function are defined in the proof of \ref{condact}. Finally, we choose $\Theta(\lambda) = ( w^{\textbf{P}(z)}, b^{\textbf{P}(z)})$ whenever the path $\lambda$ reflects the activation pattern $\mathbf{P}(z)$; meaning that at every edge of the tree $ x \in H(e_1(a_{i,j}), e_2(a_{i,j})) \iff a_{i,j} \in \lambda \iff \textbf{P}^{(j)}(z)_i = 1$, and $ x \notin H(e_1(a_{i,j}), e_2(a_{i,j})) \iff a^{(0)}_{i,j} \in \lambda \iff \textbf{P}^{(j)}(z)_i = 0$. $\mathcal{M}$ is then determined from the definition and is by construction equal to $\mathcal{N}$ everywhere. 
\end{proof}
\section{Proof of Local SHAP Values}
\begin{proof}
For a given activation pattern $\mathbf{P}(z)$, if $x, \bar{x} \in \mathcal{R}^{\mathbf{P}(x)}$, this implies that the marginal contribution for a given coalition is given by: 
\begin{align*}
    \Delta (i,S) & = f(x_{S\cup \{i\}}, x_{\overline{S\cup \{i\}}} ) - f(x_{S}, x_{\overline{S}})\\
    & = \eta^{\mathbf{P}(z)}(x_{S\cup \{i\}}, x_{\overline{S\cup \{i\}}} ) - \eta^{\mathbf{P}(z)}(x_{S}, x_{\overline{S}}) \\
\end{align*}
which results to the Shapley values of a linear model, given in \cite{lundberg2017unified}. 
\end{proof}

\section{Statement and Proof of Global SHAP Values}
In general, the masked value will not fall in the same activation region as the sample of interest. Most of the time, it is likely that masking a value will send it to a different activation region. This informs the proof of the next, more general, result. 

\begin{theorem}[\textbf{Global Shapley Values of a ReLU Neural Network}] \label{thm:shap_global}
Given a neural network $\mathcal{N}: \mathbb{R}^n \rightarrow \mathbb{R}^m$, with hyperparameters $L$, the number of layers and $\mathbf{N}= [n_1,...,n_L]$ the number of neurons for each layer, given a Linear Local Model decomposition with linear regions $\eta^{\mathbf{P}(z)}(x)$ for $x \in \mathcal{R}^{\mathbf{P}(z)}$ with activation pattern $\mathbf{P}(z)$, the global Shapley value is given by: 

\[
\phi_f(i)_j 
= \sum_{S \subset \mathcal{P}([n]/i)} \frac{(n-|S|-1)!|S|!}{N!} [ b_j^{\mathbf{P}(x^S)}- b_j^{\mathbf{P}(\bar{x}^S)} + w_{i,j}^{\mathbf{P}(x^S)}x^S_{i} - w_{i,j}^{\mathbf{P}(\bar{x}^S)}\bar{x}^S_{i} 
\]

\[+ \sum_{k\in S}(w_{k,j}^{\mathbf{P}(x^S)} - w_{k,j}^{\mathbf{P}(\bar{x}^S)} )x^S_{k} + \sum_{k \in \overline{S}/\{i\}}(w_{k,j}^{\mathbf{P}(x^S)} - w_{k,j}^{\mathbf{P}(\bar{x}^S)} )\bar{x}^S_{k}  ],
\]
where
\[\eta^{\mathbf{P}(z)}(x) = w^{\mathbf{P}(z)T} x + b ^{\mathbf{P}(z)}\]

is the Local Linear Model for the activation region $ \mathcal{R}^{\mathbf{P}(z)}$ and $x^S$ is the vector with $\overline{S \cup \{i\}}$ masked out and $\bar{x}^S$ is the vector with $\overline{S}$ masked out, and $i \in [n], j \in [m]$. 
\end{theorem}

\begin{proof}
The theorem follows form substituting the respective linear models for each marginal contribution of a Shapley value. This entails that the marginal contribution can be written as

\begin{align*}
    \Delta_f(i,S)_j & = f(x_{S\cup \{i\}}, x_{\overline{S\cup \{i\}}} )_j - f(x_{S}, x_{\overline{S}})_j\\
    & = \eta^{\mathbf{P}(x^S)}(x^S)_j - \eta^{\mathbf{P}(\bar{x}^S)}(\bar{x}^S)_j\\
    & = \sum_{k \in [n]} \left( w_{k,j}^{\mathbf{P}(x^S)} x^S_k \right) + b_j^{\mathbf{P}({x}^S)} - \sum_{k \in [n]} \left( w_{k,j}^{\mathbf{P}(\bar{x}^S)} \bar{x}^S_k \right) - b_j^{\mathbf{P}(\bar{x}^S)} \\
    & = b_j^{\mathbf{P}({x}^S)}- b_j^{\mathbf{P}(\bar{x}^S)} + \sum_{k \in S} \left( w_{k,j}^{\mathbf{P}(x^S)} x^S_k \right) +  \sum_{k \in \overline{S}/{i}} \left( w_{k,j}^{\mathbf{P}(x^S)} x^S_k \right) + w_{i,j}^{\mathbf{P}(x^S)} x^S_i \\ & - \sum_{k \in S} \left( w_{k,j}^{\mathbf{P}(\bar{x}^S)} \bar{x}^S_k \right)- \sum_{k \in \overline{S}/\{i\}} \left( w_{k,j}^{\mathbf{P}(\bar{x}^S)} \bar{x}^S_k \right) -w_{i,j}^{\mathbf{P}(\bar{x}^S)} x^S_i.
\end{align*}
Since $x_k = \bar{x}_k, \forall k \neq i$, we collect the terms to get 

\begin{align*}
    \Delta_f(i,S)_j
    & = b_j^{\mathbf{P}(x^S)}- b_j^{\mathbf{P}(\bar{x}^S)} + w_{i,j}^{\mathbf{P}(x^S)}x^S_{i} - w_{i,j}^{\mathbf{P}(\bar{x}^S)}\bar{x}^S_{i}  \\ & + \sum_{k\in S}(w_{k,j}^{\mathbf{P}(x^S)} - w_{k,j}^{\mathbf{P}(\bar{x}^S)} )x^S_{k} + \sum_{k \in \overline{S}/\{i\}}(w_{k,j}^{\mathbf{P}(x^S)} - w_{k,j}^{\mathbf{P}(\bar{x}^S)} )\bar{x}^S_{k}.
\end{align*}

Averaging over $S \in \mathcal{P}([n])/\{i\}$ ends the proof.  
\end{proof}

\end{document}